\documentclass{article}
\usepackage{graphicx,fullpage} 

\usepackage[utf8]{inputenc} 
\usepackage[T1]{fontenc}    
\usepackage{hyperref}       
\usepackage{url}            
\usepackage{booktabs}       
\usepackage{amsfonts}       
\usepackage{nicefrac}       
\usepackage{microtype}      
\usepackage{amsmath}
\usepackage{multirow}
\usepackage{amsthm, amssymb, adjustbox}
\usepackage{bm}
\usepackage{graphicx}
\usepackage{xparse}
\usepackage{anyfontsize}
\usepackage{thmtools} 
\usepackage{thm-restate}
\hypersetup{
    colorlinks=true,
    linkcolor=Blue,
    filecolor=magenta,      
    urlcolor=Cerulean,
    citecolor=blue,
    pdftitle={Overleaf Example},
    pdfpagemode=FullScreen,
    }

\usepackage{wrapfig}
\usepackage{cleveref}
\usepackage[dvipsnames]{xcolor}
\usepackage{subcaption}
\usepackage{algorithm}
\usepackage{algorithmic}
\usepackage{xcolor,colortbl}
\usepackage{pifont}
\newcommand{\xmark}{\ding{55}}%

\newtheorem{theorem}{Theorem}

\newtheorem{lemma}{Lemma}
\newtheorem{assumption}{Assumption}

\newtheorem{definition}{Definition}[section]
\usepackage{enumitem}
\usepackage{thmtools,thm-restate}

\theoremstyle{remark}
\newtheorem*{remark}{Remark}

\newcommand{\samp}{\mb{\pi}_{\mathsf{samp}}}
\newcommand{\bmu}{\mathbf{\mu}}
\newcommand{\coeff}{\mathrm{SEC}_{\mathrm{HybRLHF}}}
\newcommand{\coefffull}{\coeff(\Pi,T,\beta, \pi_{\rm samp}; \gamma, \mathcal{D}_{\rm off})}

\newcommand{\cA}{{\mathcal A}}

\newcommand{\pistarb}{\pistar_{\beta}}

\newcommand{\piref}{\pi_{\rm ref}}

\newcommand{\taut}{\Tilde{\tau}}

\newcommand{\Rmax}{R_{\rm max}}
\newcommand{\Vmax}{V_{\rm max}}

\renewcommand{\emptyset}{\varnothing}

\newcommand{\M}[1]{^{{\scriptscriptstyle M}}}  %

\newcommand{\pistar}{\pi^{\star}}

\newcommand{\approxleq}{\lesssim}

\newcommand{\ind}[1]{^{{\scriptscriptstyle(#1)}}}

\def\multiset#1#2{\ensuremath{\left(\kern-.3em\left(\genfrac{}{}{0pt}{}{#1}{#2}\right)\kern-.3em\right)}}

\renewcommand{\emptyset}{\varnothing}

\DeclareMathOperator*{\argmin}{argmin}

\DeclareMathOperator*{\amax}{\mathrm{argmax}}
\DeclareMathOperator*{\E}{\mathbb{E}}
\newcommand{\mb}{\mathbb}

\DeclareMathOperator*{\argmax}{arg\,max}
\newcommand{\add}[1]{\textcolor{cyan}{}}
\newcommand{\lindb}{\texttt{linDB}}
\newcommand{\linb}{\texttt{linB}}

\newcommand{\nangle}[1]{\langle #1\rangle}

\renewcommand{\a}{{\mathbf a}}
\renewcommand{\b}{{\mathbf b}}

\newcommand{\w}{{\mathbf w}}

\newcommand{\R}{{\mathbb R}}
\newcommand{\cD}{{\mathcal D}}

\usepackage[round, sort&compress]{natbib} 
\title{Hybrid Preference Optimization for Alignment: Provably Faster Convergence Rates by Combining Offline Preferences with Online Exploration}
\author{
Avinandan Bose\\ University of Washington\\
avibose@cs.washington.edu \and Zhihan Xiong\\ University of Washington\\
zhihanx@cs.washington.edu  \and Aadirupa Saha \\ Apple \\ aadirupa@apple.com \\ \and Simon Shaolei Du \\ University of Washington \\ ssdu@cs.washington.edu \and Maryam Fazel\\University of Washington\\mfazel@uw.edu 
}
\date{}
\begin{document}
\maketitle

\begin{abstract}
    Reinforcement Learning from Human Feedback (RLHF) is currently the leading approach for aligning large language models with human preferences. Typically, these models rely on extensive offline preference datasets for training. However, offline algorithms impose strict concentrability requirements, which are often difficult to satisfy.
    On the other hand, 
    while online algorithms can avoid the concentrability issue, pure online exploration could be expensive due to the active preference query cost and real-time implementation overhead. 
    In this paper, we propose a novel approach: Hybrid Preference Optimization (HPO) which combines online exploration with existing offline preferences by relaxing the stringent concentrability conditions for offline exploration, as well as significantly improving the sample efficiency for its online counterpart.
    We give the first provably optimal theoretical bound for Hybrid RLHF with preference feedback, providing sample complexity bounds for policy optimization with matching lower bounds. Our results yield improved sample efficiency of hybrid RLHF over pure offline and online exploration.
\end{abstract}

\section{Introduction}
\label{sec:introduction}
Reinforcement Learning from Human Feedback (RLHF) stands out as the primary method for aligning large language models with human preferences \citep{christiano2017deep, bai2022training, ouyang2022training}. Instead of starting from scratch with unsupervised training on extensive datasets, RLHF aligns pre-trained models using labeled human preferences on pairs of responses, offering a statistically lightweight approach to making language models more human-like. While labeling response pairs is easier than generating new responses, the volume of these pairs is critical for effective alignment. A large dataset is needed to ensure broad coverage of linguistic nuances, reduce the impact of noisy human feedback, and provide enough statistical power for the model to generalize well. Although labeling individual pairs is simpler, scaling this process can still become resource-intensive, making the volume of response pairs a key factor in successful model alignment. In the light of this, recently a theoretical question of interest has arisen: \textit{How can algorithms be designed to be sample-efficient during this alignment phase?}

Two main approaches have emerged in addressing this question: online RLHF and offline RLHF. Online methods \citep{xie2024exploratory, cen2024value, zhang2024self} have interactive access to human feedback or leverage a more powerful language model to explore diverse and novel responses beyond what the pre-trained model can provide. Online exploration, though theoretically sample-efficient, is costly because it requires frequent human feedback, real-time updates to large language models, and continuous deployment to generate novel responses. These frequent model updates and the need for real-time interaction with human annotators make the process resource-intensive and time-consuming, presenting significant practical challenges in terms of both scalability and cost.

In contrast, offline RLHF relies on large, readily available sets of labeled response pairs to align the language model. While these datasets are easier to collect and use, achieving provable guarantees on the model's optimality requires stringent conditions—specifically, that the dataset is close to data generated by an optimally-aligned model. 
Offline datasets, however, are static and often lack the diversity and novel responses that online exploration provides. This limitation, coupled with the assumption that the dataset contains enough near-optimal examples, restricts the model's ability to generalize effectively, potentially leading to suboptimal alignment.

This prompts the question: \textit{Can we combine offline preference data while relaxing these stringent conditions to significantly reduce the number of samples required in online exploration, thereby addressing its practical costs effectively?}

In this paper, we answer this question affirmatively.

\begin{table*}[ht]
\centering
  \begin{tabular}{l|c|c|c|c|c}
  & Category & Upper Bound & Lower Bound & Scalable & Model-free \\
  \midrule
  \citet{xie2021policy} & Traditional RL & \checkmark\cellcolor{green!25} & \checkmark\cellcolor{green!25} & \xmark\cellcolor{red!25} & \xmark\cellcolor{red!25} \\
  \hline
  \citet{tan2024hybrid} & Traditional RL & \checkmark\cellcolor{green!25} & \checkmark\cellcolor{green!25} & \xmark\cellcolor{red!25} & \xmark\cellcolor{red!25} \\
  \hline
  \citet{xiong2024iterative} & RLHF & \checkmark\cellcolor{green!25} & \xmark\cellcolor{red!25} & \xmark\cellcolor{red!25} & \checkmark\cellcolor{green!25} \\
  \hline
  \citet{gao2024rebel}  & RLHF & \checkmark\cellcolor{green!25} & \xmark\cellcolor{red!25}  & \checkmark\cellcolor{green!25} & \checkmark\cellcolor{green!25} \\
  \hline
  \citet{chang2024dataset} & RLHF & \checkmark\cellcolor{green!25}  & \xmark\cellcolor{red!25} & \checkmark\cellcolor{green!25} & \checkmark\cellcolor{green!25} \\
  \hline
  Our work  & RLHF & \checkmark\cellcolor{green!25} &\checkmark\cellcolor{green!25} &\checkmark\cellcolor{green!25} & \checkmark\cellcolor{green!25} \\
\end{tabular}
\caption{A comparison with lines of work closest to ours.} 
\label{tab:comparison_table}
\vskip -0.2in
\end{table*}

\subsection{Contributions}
Here, we list our technical contributions:
\begin{itemize}
    \item We introduce the first Hybrid RLHF algorithm that is both provably efficient and practical for implementation. In contrast, previous discussions of hybrid training in \citet{xiong2024iterative}, \citet{gao2024rebel}, and \citet{chang2024dataset} either lack practical applicability or do not provide a theoretical analysis of the advantages of hybrid training. See Table~\ref{tab:comparison_table} for a detailed comparison.

    \item Our algorithm demonstrates superior sample complexity compared to both online and offline RLHF approaches (Theorem \ref{thm:main_general}). Specifically, in the case of linear MDPs, our method surpasses the limitations imposed by the lower bounds of purely online and offline training (Theorems \ref{thm:lboff} and \ref{thm:lbon}). 
\end{itemize}

\subsection{Paper Organization}
The remaining of the paper is organized as the following. We review the related work in Section \ref{sec:related_work} and briefly introduce the background of MDPs and RLHF in Section \ref{sec:prelim}. We then present our algorithm Hybrid Preference Optimization (HPO) and its theoretical guarantee in Section \ref{sec:hpo}. Section \ref{sec:lb} examines the special case of linear MDPs and compare our results with its lower bounds in both online and offline training. We present our empirical findings in Section \ref{sec:experiments} following by a conclusion in Section \ref{sec:conclusion}.
\section{Related Work}
\label{sec:related_work}

\paragraph{DPO-related RLHF.} 
The Direct Preference Optimization (DPO) algorithm, first introduced by \citet{rafailov2024direct}, has gained considerable attention due to its effectiveness and simplicity. Since its introduction, various DPO variants have been proposed, each with unique enhancements aimed at improving performance. Several of these methods focus on modifying the loss function, such as SLiC \citep{zhao2023slic}, RSO \citep{liu2023statistical}, P3O \citep{wu2024pairwise}, PCO \citep{xu2023some}, and SimPO \citep{meng2024simpo}. Other approaches take a broader view by incorporating general preference modeling, integrating social choice theory, and introducing algorithms like Nash-MD \citep{munos2023nash}, DNO \citep{rosset2024direct}, KTO \citep{ethayarajh2024kto}, and IPO \citep{azar2024general}.

DPO's popularity has also captured the interest of the theoretical research community. \citet{xiong2024iterative}. explored the benefits of online iterative training with an offline warm-up phase, while \citet{ye2024theoretical}. extended this analysis from a reward-based framework to a general preference oracle setting. Both studies, however, limit their analysis to pure bandit environments. \citet{rafailov2024r} further reinterpreted DPO through the lens of implicit $Q^\star$ estimation within KL-constrained deterministic token-level Markov Decision Processes (MDPs). \citet{gao2024rebel} introduced a study on relative reward modeling, while recent work by \citet{cen2024value} and \citet{xie2024exploratory} analyzed the sample efficiency of combining DPO with optimistic exploration. This research has also inspired the development of the online exploration component in our Hybrid Preference Optimization (HPO) algorithm.

Despite these advancements, none of these has deep dive into a hybrid training setting. \citet{cen2024value} and \citet{ye2024theoretical} have discussed offline and online learning separately, but neither has explored the potential benefits of combining the two. \citet{xiong2024iterative} addressed hybrid training but relied on impractical assumptions, such as coverage over the optimal policy $\pi^\star$, and failed to provide a quantitative characterization of hybrid RLHF compared to purely online or offline approaches. Similarly, while \citet{gao2024rebel} mentioned hybrid training in their work, they did not delve deeply into the topic, leaving an open area for further investigation. Finally, although the same algorithm name \textit{Hybrid Preference Optimization} is also used in \citet{song2024understanding}, that version is essentially an offline DPO regularized by unlabeled online data, which is very different from our approach of enhancing online exploration through additional offline data.

 \paragraph{Other Theoretical Work in RLHF} 
 Other theoretical work on RLHF primarily centers around learning reward models and analyzing sample complexities for specific classes of Markov Decision Processes (MDPs). Among them, \citet{zhu2023principled} and \citet{zhu2024iterative} have focused on learning various reward models, while others, including \citet{du2024exploration}, \citet{zhan2023provable}, and \citet{wu2023making}, have addressed on sample complexities under tabular or linear function approximations. Some studies, such as those by \citet{chen2022human}, \citet{wang2023rlhf}, and \citet{chang2024dataset}, also explore the use of general function approximation methods. However, these algorithms are often designed around function classes with certain complexity measures, making them less applicable in practical settings. Besides, \citet{nika2024reward} provides a detailed theoretical comparison between DPO-related RLHF and reward-learning-related RLHF. Among these studies, \citet{chang2024dataset} provides a more detailed discussion on hybrid training in RLHF. Despite this, their work shares similar limitations to \citet{xiong2024iterative}, particularly in terms of too strong assumptions on single-concentrability coefficients.
 

\paragraph{Theory of Hybrid RL.} 
\citet{kalashnikov2018scalable} empirically demonstrate that in robot manipulation, a small amount of online fine-tunning can largely improve the performance of offline training. Following this, much of the theoretical work on hybrid RL has focused on quantitatively analyzing the advantages of combining online exploration with an initial offline dataset. Key studies by \citet{xie2021policy}, \citet{song2022hybrid}, and \citet{amortila2024harnessing} have explored these benefits, often under the assumption of certain concentrability conditions related to the behavior policy. However, recent research by \citet{tan2024hybrid} removes this requirement, offering an approach that does not impose explicit constraints on the quality of the behavior policy. Building on this line of research, our work extends the benefits of hybrid RL to the RLHF framework. Importantly, we adopt the same philosophy as \citet{tan2024hybrid} by not imposing additional requirements on the offline dataset, allowing for more flexible and practical applications in RLHF scenarios.

\section{Preliminaries}
\label{sec:prelim}

Our study of RLHF is under the general reinforcement learning setup similar to \citep{xie2024exploratory} which is a strict generalization of the token-level MDP proposed in \citep{rafailov2024r}.
\subsection{Reinforcement Learning from Human Feedback}
\label{sec:rlhf}

We consider an episodic finite horizon markov decision process characterized by $\mathcal{M} = (H, \mathcal{S}, \mathcal{A}, P, r, \rho)$ where $H$ is the time horizon, $\mathcal{S}$ denotes the state space, $\mathcal{A}$ denotes the action space, $P : \mathcal{S} \times \mathcal{A} \mapsto \Delta(\mathcal{S})$ denotes the transition function, $r : \mathcal{S} \times \mathcal{A} \mapsto \mathbb{R}$ denotes the reward function, and $\rho \in \Delta(\mathcal{S})$ denotes the initial state distribution. A randomized policy $\pi : \mathcal{S} \mapsto \Delta(\mathcal{A})$ generates a trajectory $\tau = \{(s_1, a_1), \ldots, (s_H, a_H)\}$ and rewards $r_1, \ldots, r_H$ through the following procedure. In particular, the initial state is sampled as $s_1 \sim \rho$, and for all subsequent time steps $h \in [H]$, actions are sampled as $a_h \sim \pi(s_h)$, reward is received as $r_h = r(s_h, a_h)$, and $s_{h+1} \sim P(s_h, a_h)$. We use shorthand notations $\mathbb{E}_{\pi}[\cdot]$ and $\mathbb{P}_{\pi}[\cdot]$ to denote the expectation and probability of quantities of interest under this process induced by the policy $\pi$. We assume the total reward is non-negative and bounded, i.e. $\sum_{h \in [H]} r_h \in [0, \Rmax]$. For compact notation at the trajectory level, we denote $r(\tau) = \sum_{h \in [H]} r(s_h, a_h)$ and $\pi(\tau) = \prod_{h \in [H]} \pi(a_h | s_h)$. The expected reward under policy $\pi$ is defined as $J(\pi) = \mathbb{E}_{\tau \sim \pi}[r(\tau)]$. 


\paragraph{Preference Model}
In RLHF, the reward signal from a trajectory $(r_1, \dots, r_H)$ is generally unobservable. Instead, the feedback from the environment is provided through a preference oracle.
In particular, given a pair of trajectories $(\tau, \tilde{\tau})$, both having initial state $s_1$, we follow the standard assumptions as in \citep{christiano2017deep, rafailov2024direct, rafailov2024r}, that the probability of $\tau$ being preferred over $\tilde{\tau}$ follows the Bradley-Terry Model \citep{bradley1952rank}, which is defined as
\begin{align}
\label{eq:btl}
    \mathbb{P}(\tau \succ \tilde{\tau} | s_1) = \frac{\exp{(r(\tau))}}{\exp{(r(\tau))} + \exp{(r(\tilde{\tau}))}}.
\end{align}
Typically, we have access to a pre-trained policy $\piref$, and the goal of RLHF is to find the policy that can maximize $J(\pi)$ while staying not too far away from $\piref$. That is, with regularization parameter $\beta>0$, we aim to maximize the following objective:
\begin{align}\label{eq:KLobj}
    J_{\beta}(\pi) &= J(\pi) - \beta \mathbb{E}_{\pi}\left[\sum_{h \in [H]} D_{KL}(\pi(\cdot| s_h) || \piref(\cdot| s_h)) \right] \nonumber\\
    &= \mathbb{E}_{\pi}\left[r(\tau) - \beta \log \frac{\pi(\tau)}{\piref(\tau)}\right].
\end{align}
Here, we denote the optimal policy of some policy class $\Pi$ as $\pi^\star_{\beta}=\argmax_{\pi\in\Pi}J_{\beta}(\pi)$.

\paragraph{Token-level MDPs.}
Given our primary focus on language models, we are particularly interested in the \textit{token-level MDP} framework, as formulated and studied by \citet{rafailov2024r}. In this context, a language model typically generates a sequence of responses from an initial prompt. The initial prompt is treated as the initial state, \( s_1 \sim \rho \), and each generated token, \( a_h \in \mathcal{A} \), is considered an action, where \( \mathcal{A} \) represents the vocabulary. Under this formulation, the state at step \( h \) is represented as \( s_h = (s_1, a_1, \dots, a_{h-1}) \), a concatenation of the initial prompt and all tokens generated before step \( h \). The state at last step, \( s_H \), can then be interpreted as the model's complete response to the initial prompt, including the prompt itself.


Clearly, from this formulation, the transition is deterministic from $s_h$ to $s_{h+1}$ given action $a_h$ since it is simply a concatenation. This is summarized as \textit{Deterministic Contextual MDPs} (DCMDPs) in \citet{xie2024exploratory}, which proposes that the optimal policy $\pi^\star_{\beta}$ of the objective \eqref{eq:KLobj} in a DCMDP satisfies
\begin{equation}
    \label{eq:DCMDP}
    \beta\log\frac{\pi^\star_{\beta}(\tau)}{\piref(\tau)}=r(\tau) - V^\star_{\beta}(s_1),\quad \forall\tau,
\end{equation}
where $V^\star_{\beta}=\max_{\pi\in\Pi}J_\beta(\pi)$. More details of DCMDP can be found in \citet{xie2024exploratory}.

To optimize objective \eqref{eq:KLobj}, two primary approaches are available: offline RLHF and online RLHF. The choice between these methods depends on the specific requirements of the application and we will discuss both approaches in details in the following sections.

\subsection{Offline, Online and Hybrid RLHF}
\label{sec:rlhfpl}

\paragraph{Offline RLHF.} 
Offline RLHF methods are restricted to the preferences captured by the reference model $\piref$ and the offline preference dataset $\mathcal{D}_{\rm off}$. Typically the offline preference dataset is a set of labelled pairs $\mathcal{D}_{\rm off} = \{(\tau_{+}^{(i)}, \tau_{-}^{(i)})\}_{i \in [N_{\rm off}]}$. To learn a policy from this offline dataset, one of the most popular approaches is Direct Preference Optimization (DPO) introduced in \citet{rafailov2024direct}, which also serves as a starting point of our work.
DPO is motivated by a closed-form solution for the policy that optimizes the KL-regularized objective in Eq.~\eqref{eq:KLobj}, and condenses the two-step process above into a single policy optimization
objective, removing the need for reward function estimation. Concretely, DPO solves the following problem:
\fontsize{8.5}{8.5}
\begin{align}\label{eq:dpo}
    \widehat{\pi} = \argmin_{\pi \in \Pi} \sum_{(\tau_{+}, \tau_{-}) \in \mathcal{D}_{\rm off}} -\log \left[\sigma\left(\beta z(\tau_+, \tau_-, \piref, \pi) \right)\right],
\end{align}
\normalsize
where $z(\tau_+, \tau_-, \piref, \pi) = \log \frac{\pi(\tau_{+})}{\piref(\tau_{+})} - \log \frac{\pi(\tau_{-})}{\piref(\tau_{-})}$, $\sigma(x)=\frac{\exp(x)}{1+\exp(x)}$ is the sigmoid function and $\Pi$ is some user-specified policy class such as parameterized neural networks.


\paragraph{Online RLHF}
Although achieving great empirical success, offline DPO is inherently limited by the support of reference policy $\piref$ and offline dataset $\mathcal{D}_{\mathrm{off}}$. Therefore, people turn more attention into online DPO, which was first theoretically studied in \citet{xiong2024iterative}. Specifically, online RLHF collects preference feedback from pairs of responses generated by the language model sequentially updated during training.
Recently, \citet{xie2024exploratory} analyzed the necessity of deliberate exploration in online DPO and proposed XPO as a combination of DPO and optimistic exploration. In particular, the original DPO loss is regularized by $\alpha\sum_{i=1}^{t}\log\pi(\widetilde{\tau}^{(i)})$ in XPO to encourage the policy to be more diverse, where $\alpha>0$ is the regularization strength and $\widetilde{\tau}^{(i)}$ is the trajectory generated by $\piref$. It was analyzed under the framework of DCMDP and achieved $\widetilde{O}(1/\sqrt{T})$ optimality gap under appropriate assumptions, where $T$ represents the number of queried preference labels. Concurrently, similar theoretical guarantee was also obtained by \citet{cen2024value}.

\paragraph{Hybrid RLHF}
Although the efficiency of online RLHF is justified both theoretically and experimentally, in application, real-time query of labels can be expensive while a large amount of offline preference dataset can often be easily and cheaply obtained. Therefore, it is natural to consider a hybrid training scheme which can potentially combine the advantage of both offline and online RLHF.
From a theoretical perspective, the sample complexity of pure online RLHF depends solely on the coverage provided by $\piref$, meaning that the larger the coverage provided by $\piref$, the smaller the space to search for online RLHF, thereby requiring fewer samples.
However, pure online RLHF neglects the coverage provided by the offline preference dataset which can significantly reduce the search space. Therefore, we propose to use online feedback to explore trajectories that are not supported in $\piref$ while still taking advantage of the available offline dataset to speed up exploration. Informally, even starting with an offline dataset on which any offline algorithm may give arbitrarily poor policy, our proposed algorithm HPO will be able to use it to significantly reduce the sample complexity, as compared to XPO.

\section{Hybrid Preference Optimization}
\label{sec:hpo}
We now present our algorithm Hybrid Preference Optimization (HPO), which integrates the available offline dataset with online exploration. This hybrid approach leverages the strengths of both offline data and online feedback, aiming to enhance the learning process and achieve a near optimal policy using fewer samples than pure online and pure offline methods.

\subsection{Algorithm Description}
We present Hybrid Preference Optimization (HPO) in Algorithm~\ref{alg:hybrid_rlhf}. Given a policy class $\Pi$, an offline dataset $\mathcal{D}_{\rm off}$, the algorithm proceeds through a series of steps over a predefined number of iterations $T$. The algorithm begins with 2 policies $\pi^{(1)}, \tilde{\pi}^{(1)}$ and iteratively updates these policies as described next. 

At each step $t \in [T]$ of the online exploration, a context is sampled as $s_1^{(t)} \sim \rho$.   The algorithm generates 2 responses: one sampled from the current policy $\tau_t \sim \pi^{(t)}(\cdot| s_1^{(t)})$, and one sampled from the model $\tilde{\tau}_t \sim \tilde{\pi}^{(t)}(\cdot| s_1^{(t)})$. This response pair is then labelled as $(\tau_{+}^{(t)}, \tau_{-}^{(t)})$ based on human preference feedback and added to the online buffer $\mathcal{D}_{\rm on}^{(t)} = \mathcal{D}_{\rm on}^{(t)} \cup (\tau_{+}^{(t)}, \tau_{-}^{(t)})$. 

\begin{algorithm}
    \caption{Hybrid Preference Optimization}
    \begin{algorithmic}[1]
    \STATE \textbf{Input:} Offline dataset $\mathcal{D}_{\rm off}$ of size $N_{\rm off}$, sampling strategy $\pi_{\rm samp}$, hyper-parameters, $\alpha \in \mathbb{R}_{+}$, $\gamma \in \mathbb{N}$.
        \STATE \textbf{Initialize:} $\pi_t, \Tilde{\pi}_t \leftarrow \piref$, $\mathcal{D}^0_{\rm on} = \emptyset$.
        \FOR{$t = 1, \ldots, T$}
        \STATE \textbf{Generate response pair:} $s_1^{(t)} \sim \rho$; $\tau^{(t)} \sim \pi^{(t)}(\cdot|s_1^{(t)}), \taut^{(t)} \sim \Tilde{\pi}_t(\cdot|s_1^{(t)})$.
        \STATE \textbf{Label with Preference:} Label $(\tau_t, \taut_t)$ as $(\tau_t^{+}, \tau_t^{-})$ based on preference feedback. 
        \STATE \textbf{Update Online Dataset:} $\mathcal{D}^{(t)}_{\rm on} \gets \mathcal{D}^{(t-1)}_{\rm on} \cup (\tau_{+}^{(t)}, \tau_{-}^{(t)})$.
        \STATE \textbf{Update Offline dataset minibatch:} Sample $\gamma$ pairs from $\mathcal{D}_{\rm off}$ uniformly randomly with replacement $\mathcal{D}_{\rm off}^{(t)}$.
        \STATE \textbf{Update Hybrid dataset:} $\mathcal{D}_{\rm hyb}^{(t)} \gets \mathcal{D}_{\rm on}^{(t)} \cup \mathcal{D}_{\rm off}^{(t)}$.
        \STATE \textbf{Update Optimism Dataset: } Compute $\mathcal{D}^{(t)}_{\rm opt}$ of $t + \gamma$ samples from $\Tilde{\pi}_t$.
        \STATE \textbf{Update Policy:} Update $\pi_t$ to maximize likelihood of preferences seen so far + regularization terms
        \fontsize{9.5}{9.5}
        \begin{align}\label{eq:HPO}
    &\pi^{(t+1)} = \argmin_{\pi \in \Pi} \bigg[\alpha \sum_{\tau \in \mathcal{D}^{(t)}_{\rm opt}} \log \pi(\tau) - \sum_{(\tau^{+}, \tau^{-}) \in \mathcal{D}^{(t)}_{\rm hyb}}  \log \bigg[ \sigma(\beta \log \frac{\pi(\tau^{+})}{\pi_{\rm ref}(\tau^{+})} - \beta \log \frac{\pi(\tau^{-})}{\pi_{\rm ref}(\tau^{-})})\bigg]\bigg].
    \end{align}\normalsize
        \STATE \textbf{Update sampling policy:} \\$\Tilde{\pi}_{t+1} \gets \pi_{\rm samp}(\pi^{(1)}, \ldots, \pi^{(t+1)})$.
        \ENDFOR
        \STATE \textbf{Return} $\widehat{\pi} = \amax_{\pi \in \{\pi^{(1)}, \ldots, \pi^{(T)}\}} J_{\beta}(\pi)$.
    \end{algorithmic}
    \label{alg:hybrid_rlhf}
\end{algorithm}

A key parameter in HPO is $\gamma \in \mathbb{N}$, which determines the number of labeled pairs drawn from the offline dataset $\mathcal{D}_{\rm off}^{(t)}$. The hybrid preference dataset at time $t$ is the union of the online buffer and $\gamma$ pairs sampled from the offline dataset $\mathcal{D}_{\rm hyb}^{(t)} = \mathcal{D}_{\rm on}^{(t)} \cup \mathcal{D}_{\rm off}^{(t)}$. An optimism dataset $\mathcal{D}_{\rm opt}^{(t)}$ is built by sampling $t + \gamma$ samples from $\tilde{\pi}^{(t)}$. 

Now, we update the policy via an optimistic variant of DPO, and similar to XPO as in Eq.~\eqref{eq:HPO}. Here $\alpha$ is an optimism parameter, and the first term in Eq.~\eqref{eq:HPO} tries to encourage the policy to explore more diverse responses compared to those generated by $\tilde{\pi}^{(t)}$ and captured by the samples in $\mathcal{D}_{\rm opt}^{(t)}$. The second term in Eq.~\eqref{eq:HPO} is the DPO objective on the hybrid preference dataset $\mathcal{D}_{\rm hyb}^{(t)}$ which aligns the policy to the human preferences. 

If $\gamma = 0$, the algorithm is exactly the same as the XPO in \citep{xie2024exploratory} and it discards all the available offline data during training. We will later show in Theorem~\ref{thm:main_general} how the choice of $\gamma$ shows up in our sample complexity bounds. Intuitively the value of $\gamma$, biases the policy to the behavior of the offline dataset. A too large $\gamma$ will put a large emphasis on the offline dataset and may hinder exploration, while a too small $\gamma$ 
 will fail to effectively utilize the available information in the offline dataset. Thus, $\gamma$ needs to be chosen in accordance to the online exploration budget $T$.

 Finally, we can see that HPO is very friendly to application as it is rooted from the vanilla DPO and the extra regularization term and sampling procedure can be easily implemented.

\subsection{Assumptions}
To establish sample complexity guarantees for HPO, we adopt standard statistical assumptions. The first of these assumptions requires that the policy class $\Pi$ be sufficiently expressive to represent the optimal KL-regularized policy.
\begin{assumption}[Policy realizability]
  \label{ass:realizability}
  The policy class $\Pi$ contains the optimal policy, i.e., $\pistarb\in\Pi$.
\end{assumption}
Policy realizability is a standard assumption for sample-efficient reinforcement learning which can be found in prior works such
\citep{agarwal2019reinforcement,lattimore2020bandit,foster2023foundations}. It is equivalent to a form of reward/value realizability as discussed in \citep{xie2024exploratory}. 
In our context, $\Pi$ will typically correspond to a class of language models with fixed architecture and variable weights. 

Next, we make a regularity assumption on the
policies in $\Pi$ \citep{rosset2024direct, xie2024exploratory}.
\begin{assumption}[Bounded density ratios]\citep{xie2024exploratory}
  \label{ass:vmax}
For all $\pi\in\Pi$ and trajectories
  $\tau=(s_1,a_1),\ldots,(s_H,a_H)$, it holds
  \begin{align}
    \left|\log\left(\frac{\pi(\tau)}{\piref(\tau)}\right)\right| \leq\frac{\Vmax}{\beta}.
         \end{align}
\end{assumption}
As discussed in \citep{xie2024exploratory}, $\Vmax$ is measurable and controllable in practice. Specifically for log-linear policies where
  $\pi(a\mid{}s)\propto\exp({\nicefrac{f(s,a)}{\beta}})$ with some linear function $f(s, a)$, it can be shown that $\Vmax\approxleq{}\Rmax$ (recall that $\Rmax$ is the range of reward).

\subsection{Model Complexity Measures}
To measure the algorithm's convergence rate towards an optimal policy, we introduce an exploration criterion that limits how frequently the algorithm encounters novel samples where it cannot effectively distinguish between the impact of $\pi$ and $\piref$ on the objective.

First, we define a preference-based analogue of the \emph{Sequential Extrapolation Coefficient} (SEC)\footnote{This is also known as an Eluder coefficient or decoupling coefficient \citep{zhong2022gec,ye2024theoretical}.} from \citep{xie2023role} that corresponds to the hybrid RLHF case. Particularly, define
\fontsize{9}{9}
\begin{align*}
    g^{(\pi)}(\tau, \taut) = \left[\beta \log \frac{\pi(\tau)}{\piref(\tau)} - r(\tau) - \beta \log \frac{\pi(\taut)}{\piref(\taut)} + r(\taut) \right]
\end{align*}
\normalsize
as a measure of the difference in the objective from Eq.~\eqref{eq:KLobj} for an arbitrary policy $\pi \in \Pi$ and two arbitrary trajectories $\tau$ and $\taut$.

For a pair of policies $\pi$ and $\Tilde{\pi}$, we define $\pi\otimes\Tilde{\pi}$ as
the joint policy that, given $s_1$, samples $\tau\sim{}\pi\mid{}s_1$
and $\taut \sim{} \Tilde{\pi} \mid{}s_1$. We write
$(\tau,\taut)\sim{}\pi\otimes\Tilde{\pi} \mid{}s_1$ as a shorthand for this process. We further let $\Tilde{\pi}\ind{t}= \samp(\pi\ind{1},\ldots,\pi\ind{t})$, and $\bmu\ind{t} = \frac{1}{t-1}\sum_{i=1}^{t-1} \pi\ind{i}\otimes \Tilde{\pi}\ind{i}$, with the convention that $\mu\ind{1}$ is arbitrary and $\samp$ is some sampling strategy such as $\mathsf{unif}(\cdot)$. 
We first define the online sequential exploration coefficient for online RLHF \citep{xie2024exploratory}.

\begin{definition} \citep{xie2024exploratory} Given a policy class $\Pi$, reference policy $\piref$, sampling strategy $\pi_{\rm samp}$, entropy regularization parameter $\beta > 0$, online exploration budget $T \in \mathbb{N}$, we define the sequential exploration coefficient (SEC) as: 
\fontsize{8.5}{8.5}
\begin{align}
    &{\rm SEC}_{\rm RLHF}(\Pi, T, \beta, \pi_{\rm samp}) = \sup_{\pi^{(1)}, \ldots, \pi^{(T)} \in \Pi}  \sum_{t=1}^T \frac{\left(\mathbb{E}_{\substack{s_1 \sim \rho, \\ \tau \sim \pi^{(t)} | s_1, \\ \taut \sim \tilde{\pi}^{(t-1)}| s_1}}[g^{(\pi^{(t)})}(\tau, \taut)]\right)^2}{\Vmax^2 \vee \left[(t - 1) \cdot \mathbb{E}_{\substack{s_1 \sim \rho, \\ (\tau, \taut) \sim \mu^{(t)} | s_1}} [\left(g^{(\pi^{(t)})}(\tau, \taut)\right)^2 ] \right]}.
\end{align}
\normalsize
\end{definition}

Next, we introduce another quantity to capture the coverage of the offline dataset in terms of any arbitrary policy $\pi \in \Pi$:
\begin{align}
    C^{(\pi)}_{\rm off} = \frac{1}{N_{\rm off}}\sum_{(\tau_{+}, \tau_{-}) \in \mathcal{D}_{\rm off} }\left(g^{(\pi)}(\tau_{+}, \tau_{-})\right)^2.
\end{align}
Below we define the SEC coefficient in the Hybrid RLHF case:
\begin{definition} Given a policy class $\Pi$, reference policy $\piref$, offline preference dataset $\mathcal{D}_{\rm off}$, offline sampling parameter $\gamma \in \mathbb{N}$, sampling strategy $\pi_{\rm samp}$, entropy regularization parameter $\beta > 0$, online exploration budget $T \in \mathbb{N}$, we define the sequential exploration coefficient (SEC) as: 
\begin{align}
   &\coefffull=\sup_{\pi^{(1)}, \ldots, \pi^{(T)}\in \Pi}   \qquad \sum_{t=1}^T \frac{\left(\mathbb{E}_{\substack{s_1 \sim \rho,\\ \tau \sim \pi^{(t)} | s_1,\\ \taut \sim \tilde{\pi}^{(t-1)}| s_1}}[g^{(\pi^{(t)})}(\tau, \taut)]\right)^2}{\Vmax^2 \vee \left[(t - 1) \cdot \mathbb{E}_{\substack{s_1 \sim \rho,\\ (\tau, \taut) \sim \mu^{(t)} | s_1}}[(\widetilde{g}_{\rm off}^{(\pi^{(t)}}))^2 ] \right]},
\end{align}
where $(\widetilde{g}_{\rm off}^{(\pi)})^2 = \left(g^{(\pi)}(\tau, \taut)\right)^2 + \gamma \cdot C^{(\pi)}_{\rm off}$.
\end{definition}

\begin{remark}
    Note that because of the additional quantity $C_{\rm off}^{(\pi)}\geq 0$, for all $\gamma \in \mathbb{N}$ and any non-empty $\mathcal{D}_{\rm off}$, we have
    \begin{align*}
        &\coefffull < {\rm SEC}_{\rm RLHF}(\Pi, T, \beta, \pi_{\rm samp}).
    \end{align*}
    In particular, $C_{\rm off}^{(\pi)}$ is a measure of coverability for an arbitrary $\pi$, contributed by the offline dataset.
\end{remark}

\subsection{Main Results}
We state the main sample complexity result of HPO below.
\begin{restatable}{theorem}{maingeneral}\label{thm:main_general}

  Suppose Assumption \ref{ass:realizability} and \ref{ass:vmax} hold.
  For any $\beta>0$ and $T\in [N]$, if we set $\alpha=c\cdot\frac{\beta}{(\Vmax + \Rmax)e^{2 \Rmax})}\cdot\sqrt{\frac{\log(|\Pi|T\delta^{-1})\log(T)}{(T + \gamma)\cdot \coefffull}}$
  for some absolute constant $c>0$,
then Algorithm \ref{alg:hybrid_rlhf} ensures that
  with probability at least $1-\delta$, we have
  \begin{align*}
    & J_{\beta}(\pi^\ast_\beta) - J_{\beta}(\widehat{\pi})
    \approxleq
(\Vmax+\Rmax)e^{2\Rmax}\cdot \quad\sqrt{\frac{(1 + \gamma / T)\cdot\mathsf{SEC}\cdot\log(|\Pi|T\delta^{-1})\log(T)}{T}},
  \end{align*}
  where $\mathsf{SEC} = \coefffull$.
  \end{restatable}
Note that $\gamma=0$ represents the pure online setting, which is studied in XPO \citep{xie2024exploratory}. Here, we can notice that the sample complexity is reduced because the SEC for hybrid RLHF becomes smaller than its pure online counterpart. That is, the size of the space that needs to be searched through online exploration becomes smaller because of the available offline dataset. 

\section{Breaking Lower Bounds through Hybrid Learning} \label{sec:lb}

In order to paint a clearer picture of why Hybrid RLHF is more sample-efficient than pure online and pure offline RLHF, we consider the special case of linear MDPs and study the suboptimality gaps of hybrid, online and offline RLHF. Our main goal is to show that the upper bounds for hybrid RLHF are smaller than the minimax lower bounds for pure online and pure offline RLHF. 

We begin by formally defining the linear MDP setting.
 \begin{definition}(Linear MDP \citep{jin2020provably}) 
In a Linear MDP, the transition probability as well as the reward are linear functions of a known feature map $\phi(s,a) \in \mathbb{R}^d$, where $\|\phi\|_2 \leq 1$ and
\begin{align*} 
P(s' | s, a) = \phi(s, a)^\top \mu(s') \quad r(s, a) = \phi(s, a)^\top \nu.
\end{align*}
Here, $\mu(s')$ is an unknown feature map with $\|\sum_{s'} \mu(s')\|_2 \leq \sqrt{d}$, and $\nu \in \mathbb{R}^d$ is an unknown parameter with $\|\nu\|_2 \leq 1$.
\end{definition}

\subsection{Lower Bounds for Pure Offline and Online RLHF}
Define $\Lambda_{\rm off} = \frac{1}{N_{\rm off}}\sum_{(\tau, \taut) \in \mathcal{D}_{\rm off}}(\phi(\tau) - \phi(\taut))(\phi(\tau) - \phi(\taut))^\top$ as the empirical feature covariance matrix of the offline preference dataset. Let $\nu^ \ast = \frac{\mathbb{E}_{s \sim \rho, \tau \sim \pi^\ast_\beta(\cdot| s)}[\phi(s, \tau)]}{\|\mathbb{E}_{s \sim \rho, \tau \sim \pi^\ast_\beta(\cdot| s)}[\phi(s, \tau)]\|_2}$ denote a unit vector corresponding to the feature projection along the optimal policy $\pi^\ast_\beta$. We define the following quantity, for a given instance, i.e. a linear MDP $\mathcal{M} = \{\phi, \mu, \nu\}$ and offline dataset $\mathcal{D}_{\rm off}$:
\begin{align*}
    \rm{C}^{\ast}(\mathcal{M}, \mathcal{D}_{\rm off}) = \|\Lambda_{\rm off}^{-\frac{1}{2}} \nu^\ast\|_{2}.
\end{align*}
This measures the coverage of the offline dataset with respect to the optimal policy $\pi^\ast_\beta$. A Larger value indicates poorer coverage. Since for any arbitrary instance this quantity may be unbounded above, we consider all families of linear MDPs such that the concentrability is at most $\Delta$, i.e. 
\begin{align*}
    \rm{CB}(\Delta) = \{\mathcal{M}, \mathcal{D}_{\rm off} | \rm{C}^{\ast}(\mathcal{M}, \mathcal{D}_{\rm off}) \leq \Delta\}. 
\end{align*} Then by extending results from \cite{li2022pessimism}, there exists an instance in $\rm{CB}(\Delta)$ such that the suboptimality for any algorithm using $\ell_2$ confidence sets is lower bounded by :
\begin{align*}
    J_{\beta}(\pi^\ast_\beta) - J_{\beta}(\widehat{\pi}) \geq \Omega\left(\sqrt{\frac{d}{N_{\rm off}}} \Delta \right). 
\end{align*}


Thus, by choosing $\Delta = \mathcal{O}(\sqrt{d})$ even in the restricted class $\rm{CB}(\mathcal{O}(\sqrt{d}))$, the suboptimality lower bound is $\Omega\left(\sqrt{\frac{d^2}{N_{\rm off}}}\right)$. We formally state this as a minimax lower bound for the pure offline case: 
\begin{theorem}[Sample Complexity Lower Bound for Offline RLHF]
 \label{thm:lboff}
For any algorithm $\cA$, there exists an instance of offline RLHF problem and choice of $\beta$, such that if $\hat \pi_n^{\cA}$ is the policy returned by $\cA$ after any $n \geq d^2$ pairwise preference samples, then $J_\beta(\pi^*_\beta) - J_\beta(\hat \pi_n^{\cA}) \geq \Omega\left(\sqrt{\frac{d^2}{n}}\right)$.
\end{theorem}

We now state a minimax lower bound for the pure online case: 

\begin{theorem}[Sample Complexity Lower Bound for Online RLHF]
 \label{thm:lbon}
 For any algorithm $\cA$, there exists an instance of online RLHF problem and choice of $\beta$, such that if $\hat \pi_T^{\cA}$ is the policy returned by $\cA$ after $T$ rounds, then $J_\beta(\pi^*_\beta) - J_\beta(\hat \pi_T^{\cA}) \geq \Omega\left(\sqrt{\frac{d^2}{T}}\right)$.
\end{theorem}


The proof of \cref{thm:lboff} and \cref{thm:lbon} respectively follows the existing lower bounds for online \citep[Theorem 2]{wagenmaker2022reward} and offline \citep[Theorem 2]{li2022pessimism} linear contextual bandits along with the reduction to idea argued in \cite[Lemma 8]{saha2021optimal}. We have added the detailed proofs in \cref{app:lbproofs}.

\subsection{Upper Bounds for Hybrid RLHF}
Define $\Tilde{\Lambda}_{\rm off} = \Lambda_{\rm off} + \frac{\Vmax^2}{\gamma} \mathbf{I}$ as the empirical coverage.  
Let $\lambda^{(1)}_{\rm off}, \dots, \lambda^{(d)}_{\rm off}$ be the eigenvalues of $\Tilde{\Lambda}_{\rm off}$. Let $\{v_1, v_2, \ldots, v_d\}$ denote the orthonormal eigenvectors corresponding to $\Tilde{\Lambda}_{\rm off}$. We adopt the convention that the eigenvalues corresponding to these eigenvectors are in increasing order. For a fixed threshold $\lambda$ we can define:
\begin{align*}
d_{\rm hyb} = \max_{i \in [d]} \left\{\|\Tilde{\Lambda}_{\rm off}^{-\frac{1}{2}} v_i\|_{2} \approxleq \Omega\left(\frac{1}{\sqrt{T}}\right)\right\}. 
\end{align*}
$d_{\rm hyb}$ measures the coverage of the offline dataset. A larger value of $d_{\rm hyb}$ indicates poor coverage. Note that the definition of $d_{\rm hyb}$ is also equivalent to $\max_{i\in[d]}\lambda^{(i)}_{\rm off}\leq\Omega(1/T)$.
We first state an upper bound on the sample complexity of HPO.

\begin{restatable}{theorem}{linearupper}\label{thm:linear_upper}
    By choosing $\gamma = \mathcal{O}(T)$, we then have
    \begin{align*}
        J_{\beta}(\pi^\ast_\beta) - J_{\beta}(\widehat{\pi})
    \approxleq \tilde{\mathcal{O}}\left(\Rmax e^{2\Rmax} \sqrt{\frac{d\cdot d_{\rm hyb}}{T}}\right).
    \end{align*}
\end{restatable}
\begin{proof}
We highlight the key steps of this proof.
Using the elliptical potential lemma \cite{lattimore2020bandit}, we can derive
\begin{align*}
    &\coefffull \\
    &\leq 2  \log( |\gamma \Lambda_{\rm off} + (\Vmax^2 + 4T)\mathbf{I}|) - 2\log(|\gamma\Lambda_{\rm off} + \Vmax^2\mathbf{I}|)\\ 
    &=2  \log( |\Tilde{\Lambda}_{\rm off} + 4T\mathbf{I} / \gamma|) - 2\log(|\Tilde{\Lambda}_{\rm off}|)\\ 
    &= 2  \sum_{i \in [d]}\log\left(\lambda_{\rm off}^{(i)} +  4T/\gamma\right) -2\sum_{i \in [d]} \log\left(\lambda_{\rm off}^{(i)}\right)\\
     &\leq 2  \sum_{i \in [d]}\left(\log\left(1 + \frac{4T}{ \gamma \lambda_{\rm off}^{(i)}}\right)\right).
\end{align*}
Choosing $\gamma = \mathcal{O}(T)$, then only terms with $\lambda_{\rm off}^{(i)} \leq \Omega(\frac{1}{T})$ will remain in the summand, corresponding to the directions where the offline dataset has poor exploration. By defintion, $d_{\rm hyb}$ denote the number of indices in $[d]$ such that $\lambda_{\rm off}^{(i)} \leq \Omega(\frac{1}{T})$. Noting that $\gamma \lambda_{\rm off}^{(i)} \geq \Vmax^2 \quad \forall i \in [d]$, we get $\coefffull \leq \mathcal{O}\left(d_{\rm hyb}\log\left(1 + \frac{4T}{\Vmax^2}\right)\right)$.
Hence, the SEC scales as the effective number of dimensions yet to be explored sufficiently. Plugging this into Theorem~\ref{thm:main_general} gives the desired result.
\end{proof}

Note that $d_{\rm hyb}$ is always upper bounded by $d$. Thus for any arbitrary instance any arbitrary offline dataset, this upper bound is still $\tilde{\mathcal{O}}\left(\sqrt{\frac{d^2}{T}}\right)$.
Furthermore, from the above analysis, it is clear that compared with pure offline and pure online RLHF, HPO has the following two advantages:
\begin{enumerate}
    \item The upper bound of HPO in Theorem \ref{thm:linear_upper} beats the lower bounds of both pure online and offline RLHF in Theorem \ref{thm:lboff} and \ref{thm:lbon}, as long as $d_{\rm hyb}$ is non-trivially smaller than $d$. That is, our proposed hybrid RLHF is provably better than both pure online and offline RLHF.
    \item To achieve this upper bound, HPO does not require any assumptions in all-policy or single-policy concentrability coefficient, which is usually imposed in traditional offline RL \citep{zhan2022offline}. To satisfy these kinds of assumptions, the behavior policy used to collect the offline dataset needs to cover all possible actions of the optimal policy, which can sometimes be stringent. However, HPO can effectively utilize the offline dataset even if the behavior policy covers no action of the optimal policy as $d_{\rm hyb}$ shrinks even through coverage in the offline datasets collected via by suboptimal policies.
\end{enumerate}
\section{Experiments}\label{sec:experiments}
\begin{figure*}[t]
\centering
\begin{subfigure}{.5\textwidth}
  \centering
  \includegraphics[width=\linewidth]{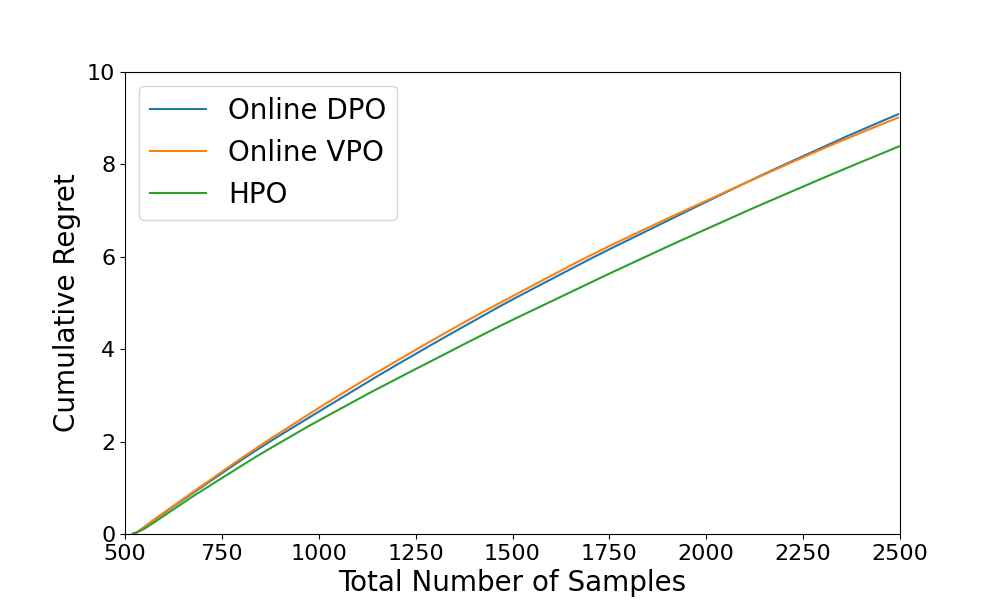}
  \caption{Cumulative Regret}
  \label{fig:sub1}
\end{subfigure}%
\begin{subfigure}{.5\textwidth}
  \centering
  \includegraphics[width=\linewidth]{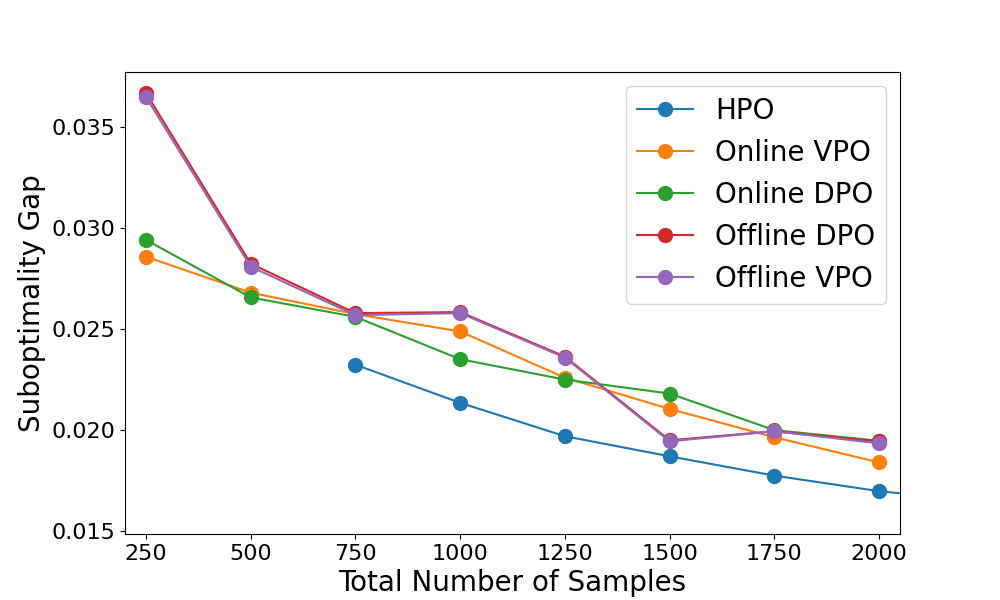}
  \caption{Suboptimality Gap}
  \label{fig:sub2}
\end{subfigure}
\caption{(Left) We plot the cumulative regret as the number of online samples grows. (Right) We plot the suboptimality gaps for all online, offline and hybrid algorithms as a function of total number of samples. Since HPO has access to offline samples, to ensure fairness while comparinng, for the hybrid case we report performance with number of offline + online samples seen. HPO used an offline dataset of size $N_{\rm off} = 500$, hence in both the plots the value corresponding to 750, is after 250 additional online samples seen, the value corresponding to 1000, is after 500 additional online samples collected by HPO and so on. }
\label{fig:test}
\end{figure*}

We evaluate the benefits of HPO on the linear contextual bandit setting considered in \citep{cen2024value}. We restate the salient features of their setup here. 
We consider a linear contextual bandit problem, where we set the prompt space as $\mathcal{X} = \mathbb{R}^2$ and the response space as $|\mathcal{Y}| = 500$ one-hot vectors. For each $(x, y)$ pair, the ground truth reward is given by $r^\star(x, y) = \langle \phi(x, y), \theta^\star \rangle$, where $\theta^\star \in \mathbb{R}^{100}$ is randomly sampled from $\mathcal{U}([0, 1])$, and the feature vector $\phi(x, y)$ is the output of the hidden layer of a fixed two-layer MLP, with the input given by the concatenation of $x$ and the one-hot encoding of $y$. The activation function is set to $\tanh$. The context vector $x$ is drawn from standard normal distribution.

We focus on log-linear policy class $\pi_\theta(\cdot|x) = \rm{softmax}(\langle \theta, \phi(x, \cdot) \rangle)$, and set $\pi_{\text{ref}} = \pi_{\theta_{\text{ref}}}$ with $\theta_{\text{ref}}(x, y)$ sampled i.i.d. from $\mathcal{U}([0, 1])$. 

\textbf{Offline Data Collection Policy:} For the offline data collection process we collected response pairs from the reference policy $\pi_{\text{ref}}$.

\textbf{Optimization Procedure:} We use mini-batch samples of size 5 in every iteration. We approximately solve the optimization problems by performing 20 AdamW optimization steps with learning rate 0.01 and weight decay rate 0.01 in every iteration for the online setting and 1000 steps for the offline setting.

\textbf{Choice of Hyperparameters:} For Online VPO, we tried hyperparameters $\alpha = \{1.0, 10.0\}$ and for offline VPO we tried hperparameters $\alpha = \{\frac{0.1}{\sqrt{N_{\rm off}}} , \frac{1.0}{\sqrt{N_{\rm off}}}, \frac{10.0}{\sqrt{N_{\rm off}}}\}$ and report the results for the best performing hyperparameter here. For HPO we try the same set of $\alpha$'s as Online VPO and set $\gamma = N_{\rm off} = 500$. 

We ask the following question:
     \textit{Starting from an offline dataset where offline algorithms like DPO \citep{rafailov2024direct}, Offline-VPO \citep{cen2024value} fail to obtain a near optimal policy, can we still utilize information from the dataset to significantly reduce the number of online exploration samples needed to obtain a near optimal policy via HPO as compared to baselines such as Online DPO \citep{guo2024direct}, Online VPO \citep{cen2024value} (XPO \citep{xie2024exploratory})?}

We answer this question affirmatively. In particular, we note the following takeaways from our experiment:
\begin{enumerate}
    \item In Figure~\ref{fig:test} (left), the cumulative regret of HPO grows much slower compared to pure online learning baseline algorithms. The value plotted for any $T$ with $T > N_{\rm off}$ is $\sum_{i = 500}^T \left(J_{\beta}(\pi^\ast_\beta) - J_{\beta}(\pi_{i})\right)$. This ensures the online algorithms have seen online $T$ samples, when we are talking about the cumulative regret at $T$, whereas the hybrid algorithm has seen $N_{\rm off}$ offline samples and $T - N_{\rm off}$ offline samples. This ensures fairness while comparing the online and hybrid approaches. 
    \item In Figure~\ref{fig:test} (right), we plot the suboptimality gaps as a function of total number of samples (online, offline or a mix of both). For the hybrid setting, we use an offline dataset of size $N_{\rm off} = 500$. Note that the suboptimality gaps for offline algorithms run on this size of a dataset is quite poor compared to other (purely online) baselines. However, using the same dataset, the suboptimality gaps for the policy produced by HPO is much better than all baselines. 
\end{enumerate}






\section{Conclusion and Future Work}
\label{sec:conclusion}
In this paper, we introduced Hybrid Preference Optimization (HPO), a novel approach that effectively integrates the strengths of both online and offline Reinforcement Learning from Human Feedback (RLHF) methods. By combining online exploration with existing offline preference data, HPO addresses the limitations inherent in each approach when used in isolation. Specifically, it relaxes the stringent concentrability conditions required by offline methods and enhances the sample efficiency of online exploration. Our theoretical analysis provides the first provably optimal bounds for hybrid RLHF with preference feedback, demonstrating that HPO achieves better sample complexity than either pure offline or pure online methods.
The empirical implications of our approach suggest a promising avenue for developing more efficient and practical methods for aligning large language models with human preferences. By leveraging the wealth of existing offline data while simultaneously adapting to new information through online exploration, HPO reduces the need for extensive and costly real-time human feedback, thereby making the process more scalable and less resource-intensive.

\section*{Acknowledgements}
MF was supported in part by awards NSF TRIPODS II 2023166, CCF 2007036, CCF 2212261, CCF 2312775. SSD acknowledges the support of NSF IIS 2110170, NSF DMS 2134106, NSF CCF 2212261, NSF IIS 2143493, NSF IIS 2229881, Alfred P. Sloan Research Fellowship, and Amazon.


\medskip

\bibliography{References}

\newpage
\appendix
\onecolumn

\section{Upper Bound Proofs}\label{app:proofs}
The proof structure for Theorem~\ref{thm:main_general} is similar to Theorem 3.1 \citep{xie2024exploratory}. However due to the hybrid nature of our algorithm, we are able to derive concentration lemmas with a faster convergence rate. We highlight the main differences in our analysis, which lead to the reduced sample complexity. 
\subsection{Concentration Lemmas}
Let $\mathbb{P}_{\rm off}$ denote the nominal distribution of trajectory pairs in the offline dataset $\mathcal{D}_{\rm off}$. Define $\bm{\mu}^{(t)}_{\rm hyb} = \frac{1}{t - 1 + \gamma} (\sum_{i=1}^{t-1} \pi^{(i)} \otimes \tilde{\pi}^{(i)} + \gamma \mathbb{P}_{\rm off})$.

Define 
\begin{align*}
    f_{\pi}(\tau, \taut) = \beta \log \frac{\pi(\tau)}{\piref(\tau)} - \beta \log \frac{\pi(\taut)}{\piref(\taut)}. 
\end{align*}
\begin{lemma}[Concentration for Algorithm \ref{alg:hybrid_rlhf}]\label{lem:concentration}
Suppose that Assumptions \ref{ass:realizability},\ref{ass:vmax} hold. Then
   Algorithm \ref{alg:hybrid_rlhf} guarantees that with probability at least
  $1-\delta$, for all steps $t\in [T]$,
  \begin{align*}
          &\alpha\cdot\E_{s_1\sim\rho,\tau\sim\Tilde{\pi}\ind{t-1}}{[\log(\pi\ind{t}(\tau))-\log(\pistarb(\tau))]} + \kappa \cdot\left[\E_{s_1\sim\rho{},(\tau,\taut)\sim{}\bm{\mu}\ind{t}_{\rm hyb}\mid{}s_1}{[f_{\pi\ind{t}}(\tau,\taut)
        - f_{\pistarb}(\tau,\taut)]^2}\right] \notag \\&\leq \frac{2\log(2|\Pi|T\delta^{-1})}{\gamma + t-1} + \frac{\alpha}{\beta}\Vmax\sqrt{\frac{2^4\log(2|\Pi|T\delta^{-1})}{\gamma + t-1}},   \end{align*}
  for $\kappa = (8(\Rmax+\Vmax)e^{2\Rmax})^{-2}$.
\end{lemma}
\begin{proof}
Fix $t \in \{2, \ldots, T+1\}$. 

We wish to decompose the objective in HPO, Eq~\eqref{eq:HPO} into 2 terms as defined below:
Define the following quantities:
\begin{align}\label{eq:HPO_LB}
    \pi^{(t)} = \argmin_{\pi \in \Pi} \bigg[\widehat{L}^{(t)}(\pi) + \widehat{B}^{(t)}(\pi)\bigg].
\end{align}
\begin{align*}
    \widehat{L}^{(t)}(\pi) = \sum_{(\tau^{+}, \tau^{-}) \in \mathcal{D}^{(t)}_{\rm hyb}} \log \bigg[ \sigma(\beta \log \frac{\pi(\tau^{+})}{\pi_{\rm ref}(\tau^{+})} - \beta \log \frac{\pi(\tau^{-})}{\pi_{\rm ref}(\tau^{-})})\bigg].
\end{align*}
\begin{align*}
    \widehat{B}^{(t)}(\pi) = \alpha \sum_{\tau \in \mathcal{D}^{(t)}_{\rm opt}} \log \pi(\tau).
\end{align*}
Now we provide concentration lemmas for these individual terms before combining them together. 
\begin{lemma}\label{lem:L} For any fixed $t \geq 1$, for all $\pi \in \Pi$ with probability at least $1 - \delta$, we have:
    \begin{align*}
        \kappa \cdot\left[\E_{s_1\sim\rho{},(\tau,\taut)\sim{}\bm{\mu}\ind{t}_{\rm hyb}\mid{}s_1}{[f_{\pi\ind{t}}(\tau,\taut)
        - f_{\pistarb}(\tau,\taut)]^2}\right] \leq \widehat{L}^{(t)}(\pi) - \widehat{L}^{(t)}(\pistarb) + \log(2|\Pi|T\delta^{-1})
    \end{align*}
\end{lemma}
\begin{proof}
    Use the definition of $\bm{\mu}^{(t)}_{\rm hyb}$ to see the dataset $\mathcal{D}_{\rm hyb}^{(t)}$ is a sequence adapted to the filtration $\mathcal{F}^{(\gamma + t)} = \sigma(\mathcal{D}_{\rm off}, \ldots, \mathcal{D}_{\rm off}, (\tau^{(1)}, \tilde{\tau}^{1}), \ldots, (\tau^{(t-1)}, \tilde{\tau}^{(t-1)}))$. Then apply the Martinagle Chernoff Theorem to get a high probability bound. Combine it with Lemma C.8 \citep{xie2024exploratory} to get the above result. 
\end{proof}
\begin{lemma}\label{lem:B} For any fixed $t \geq 1$, for all $\pi \in \Pi$ with probability at least $1 - \delta$, we have:
    \begin{align*}
    \alpha\cdot (t + \gamma -1) \cdot \E_{s_1\sim\rho,\tau\sim\Tilde{\pi}\ind{t-1}}{[\log(\pi\ind{t}(\tau))-\log(\pistarb(\tau))]} \leq \widehat{B}^{(t)}(\pi) - \widehat{B}^{(t)}(\pistarb) + \frac{\alpha}{\beta}\Vmax\sqrt{2^4\log(2|\Pi|\delta^{-1})}
\end{align*}
\end{lemma}
\begin{proof}
    Apply Azuma Hoeffding with its sample average over $t + \gamma$ terms.
\end{proof}

Combining Lemma~\ref{lem:L} and Lemma~\ref{lem:B}, taking an union bound over all time steps $t \in [T]$, and utilizing the definition of $\pi^{(t)}$ in Eq.~\eqref{eq:HPO_LB} to note $\widehat{B}^{(t)}(\pi) + \widehat{L}^{(t)}(\pi) \leq \widehat{B}^{(t)}(\pistarb) + \widehat{L}^{(t)}(\pistarb)$, get  we get the stated result.
\end{proof}
\subsection{Proof of Theorem~\ref{thm:main_general}}
We use the standard steps for regret decomposition, and similar proof idea in \citep{xie2024exploratory}. The main difference in our analysis is the definition of the term $\mathcal{I}^{(t)}$, which allows us to use the faster convergence concentration lemma in Lemma~\ref{lem:concentration}, and also use our definition of the SEC in the hybrid RLHF case, $\coefffull$, which is smaller than the SEC in the online case. 
\maingeneral*
\begin{proof}
    Using the regret decomposition techniques for KL-Regularized MDPs from \citep{xie2024exploratory}, we derive:
\begin{align*}
    &J_{\beta}(\pi_{\beta}^\ast) - J_{\beta}(\pi^{(t)}) \\
    &\leq \frac{6 \Vmax}{T} + \frac{1}{T} \sum_{t=2}^T \mathbb{E}_{\tau \sim \tilde{\pi}^{(t-1)}} \bigg[ \beta \log(\pi\ind{t}(\tau)) - \beta \log(\pistarb(\tau)) \bigg] \\
    &+ \frac{1}{T} \sum_{t=2}^T \mathbb{E}_{s_1 \sim \rho, \tau \sim \pi^{(t)}, \taut \sim \tilde{\pi}^{(t)}} \bigg[ \beta \log \frac{\pi^{(t)}(\tau)}{\pi_{\rm ref}(\tau)} -r(\tau) - \beta \log \frac{\pi^{(t)}(\taut)}{\pi_{\rm ref}(\taut)} + r(\taut) \bigg]
\end{align*}
Recalling the definition of $\bm{\mu}^{(t)}_{\rm hyb} = \frac{1}{t - 1 + \gamma} (\sum_{i=1}^{t-1} \pi^{(i)} \otimes \tilde{\pi}^{(i)} + \gamma \mathbb{P}_{\rm off})$, we define:
\begin{align*}
    \mathcal{I}^{(t)} = \frac{\left(\mathbb{E}_{s_1 \sim \rho, \tau \sim \pi^{(t)}, \taut \sim \tilde{\pi}^{(t)}} \bigg[ \beta \log \frac{\pi^{(t)}(\tau)}{\pi_{\rm ref}(\tau)} -r(\tau) - \beta \log \frac{\pi^{(t)}(\taut)}{\pi_{\rm ref}(\taut)} + r(\taut) \bigg]\right)^2}{\Vmax^2 \vee (t - 1 + \gamma) \cdot \mathbb{E}_{s_1 \sim \rho, \tau, \taut \sim \bm{\mu}^{(t)}_{\rm hyb} | s_1} \bigg[ \left(\beta \log \frac{\pi^{(t)}(\tau)}{\pi_{\rm ref}(\tau)} -r(\tau) - \beta \log \frac{\pi^{(t)}(\taut)}{\pi_{\rm ref}(\taut)} + r(\taut)\right)^2\bigg]} 
\end{align*}
Using AM-GM inequality we can bound:
\begin{align*}
    & \mathbb{E}_{s_1 \sim \rho, \tau \sim \pi^{(t)}, \taut \sim \tilde{\pi}^{(t)}} \bigg[ \beta \log \frac{\pi^{(t)}(\tau)}{\pi_{\rm ref}(\tau)} -r(\tau) - \beta \log \frac{\pi^{(t)}(\taut)}{\pi_{\rm ref}(\taut)} + r(\taut) \bigg]\\ 
    &\leq  \frac{\mathcal{I}^{(t)}}{2 \eta}  + \frac{\eta}{2}\cdot  \left( \Vmax^2 \vee (t - 1 + \gamma) \cdot \mathbb{E}_{s_1 \sim \rho, \tau, \taut \sim \bm{\mu}^{(t)}_{\rm hyb}} \bigg[ \left(\beta \log \frac{\pi^{(t)}(\tau)}{\pi_{\rm ref}(\tau) | s_1} -r(\tau) - \beta \log \frac{\pi^{(t)}(\taut)}{\pi_{\rm ref}(\taut)} + r(\taut)\right)^2\bigg]\right)\\
    &\leq  \frac{\mathcal{I}^{(t)}}{2 \eta}  + \frac{\eta}{2}\cdot  \left( \Vmax^2 + (t - 1 + \gamma) \cdot \mathbb{E}_{s_1 \sim \rho, \tau, \taut \sim \bm{\mu}^{(t)}_{\rm hyb} | s_1} \bigg[ \left(\beta \log \frac{\pi^{(t)}(\tau)}{\pi_{\rm ref}(\tau)} -r(\tau) - \beta \log \frac{\pi^{(t)}(\taut)}{\pi_{\rm ref}(\taut)} + r(\taut)\right)^2\bigg]\right)\\
\end{align*}

Recall the definition of $\coefffull$, and note that $\sum_{t=1}^T \mathcal{I}^{(t)} \leq \coefffull$.

Now we can write:
\begin{align*}
    &J_{\beta}(\pi_{\beta}^\ast) - J_{\beta}(\pi^{(t)})\\
    &\leq \frac{6 \Vmax}{T} + \frac{\coefffull}{2\eta T} + \frac{\eta}{2} \Vmax^2 + \frac{1}{T} \sum_{t=2}^T \bigg(\mathbb{E}_{\tau \sim \tilde{\pi}^{(t-1)}} \bigg[ \beta \log(\pi\ind{t}(\tau)) - \beta \log(\pistarb(\tau)) \bigg] \\
    &+ \frac{\eta}{2} \cdot (t - 1 + \gamma) \cdot \mathbb{E}_{s_1 \sim \rho, \tau, \taut \sim \bm{\mu}^{(t)}_{\rm hyb} | s_1} \bigg[ \left(\beta \log \frac{\pi^{(t)}(\tau)}{\pi_{\rm ref}(\tau)} -r(\tau) - \beta \log \frac{\pi^{(t)}(\taut)}{\pi_{\rm ref}(\taut)} + r(\taut)\right)^2\bigg]\bigg)
\end{align*}
For a fixed $t \in \{2, \ldots, T + 1\}$, we consider the term:
\begin{align*}
    &\mathbb{E}_{\tau \sim \tilde{\pi}^{(t-1)}} \bigg[ \beta \log(\pi\ind{t}(\tau)) - \beta \log(\pistarb(\tau)) \bigg]\\
    &\qquad\qquad + \frac{\eta}{2} \cdot (t - 1 + \gamma) \cdot \mathbb{E}_{s_1 \sim \rho, \tau, \taut \sim \bm{\mu}^{(t)}_{\rm hyb} | s_1} \bigg[ \left(\beta \log \frac{\pi^{(t)}(\tau)}{\pi_{\rm ref}(\tau)} -r(\tau) - \beta \log \frac{\pi^{(t)}(\taut)}{\pi_{\rm ref}(\taut)} + r(\taut)\right)^2\bigg]
\end{align*}
Setting $\eta = \frac{\beta \kappa}{\alpha (T + \gamma)}$, and invoking Lemma~\ref{lem:concentration}, we get:
\begin{align*}
    &J_{\beta}(\pi_{\beta}^\ast) - J_{\beta}(\pi^{(t)})\\
    &\approxleq \frac{\Vmax}{T} + \frac{\alpha (T + \gamma )\coefffull}{\beta \kappa T} + \frac{\beta \kappa}{\alpha(T + \gamma)} \Vmax^2 + 
    \\& \sum_{t=2}^T \frac{1}{T} \left( \frac{\beta}{\alpha}\frac{\log(|\Pi|T\delta^{-1})}{\gamma + t-1} + \Vmax\sqrt{\frac{\log(|\Pi|T\delta^{-1})}{\gamma + t-1}} \right) \\
    &\approxleq \frac{\Vmax}{T} + \frac{\alpha (T + \gamma )\coefffull}{\beta \kappa T} + \frac{\beta \kappa}{\alpha(T + \gamma)} \Vmax^2 + \\
    & \frac{\beta}{\alpha}\frac{\log(|\Pi|T\delta^{-1}) \log(T)}{T} + \Vmax\sqrt{\frac{\log(|\Pi|T\delta^{-1})}{T}}
\end{align*}
Choosing 
\begin{align*}
    \alpha \propto \frac{\beta}{(\Vmax + \Rmax)e^{2 \Rmax}}\cdot\sqrt{\frac{\log(|\Pi|T\delta^{-1})\log(T)}{(T + \gamma)\cdot \coefffull}}
\end{align*}
and noting $\kappa < \Vmax^{-2}$, we get:
\begin{align*}
    &J_{\beta}(\pi_{\beta}^\ast) - J_{\beta}(\pi^{(t)})\\
    &\leq \kappa^{-1}\sqrt{\frac{(1 + \gamma / T)\cdot\coefffull\cdot\log(|\Pi|T\delta^{-1})\log(T)}{T}} + \Vmax\sqrt{\frac{\log(|\Pi|T\delta^{-1})}{T}}\\
    &\leq \mathcal{O}((\Vmax+\Rmax)e^{2\Rmax})\cdot \sqrt{\frac{(1 + \gamma / T)\cdot\coefffull\cdot\log(|\Pi|T\delta^{-1})\log(T)}{T}}.
\end{align*}
\end{proof}

\section{Lower Bound Proofs}
\label{app:lbproofs}

In this section, we will analyze the proofs of \cref{thm:lbon} and \cref{thm:lboff}. Our proofs rely on a key observation that the RLHF problem, as described in \cref{sec:prelim} can essentially be seen as a BTL-based dueling bandit (DB) \add{cite DB} setting \citep{rc,bengs2022stochastic}, both for the online and offline problem. For completeness, we first describe the preference model below:

\textbf{BTL-based Pairwise Preference (Dueling) Model:}
Consider a decision space $\cD$. Each action/item $\a \in \cD$ is associated to a linear score/reward $s(\a) \in \R$. The probability of action $\a$ preferred over action $\b$ is given by: 
\[
P(\a \succ \b) = \big(\sigma\big(s(\a) - s(\b)\big)\big),
\]
where $\sigma(\cdot)$ being the sigmoid transformation given by $\sigma(x) = \frac{1}{1+e^{-x}}$ for any $x \in \R$. Note the above preference model exactly boils down to the BTL model-based preferences described in \cref{eq:btl} in \cref{sec:prelim}. We will denote this preference model as BTL-DB.

This essentially establishes the connection between the RLHF feedback model and BTL-based preference model BTL-DB (note that the trajectories $(\tau)$ lie in the action space $\cD$ for the DB problem). Further note that the reward/ score function $s(\cdot)$ essentially corresponds to the $J(\cdot)$ function of \cref{eq:KLobj} (for $\beta =0$). We next establish a connection between two well-studied feedback models in learning theory: \emph{linear score BTL-DB} feedback and linear bandits (\linb) feedback. We first describe the two feedback models below:

\paragraph{Linear score based BTL-DB (\lindb) Feedback model.}
In addition to the modeling assumptions described above for BTL-DB, we further assume $\cD \subset \R^d$ and the underlying score/reward function $(s)$ is linear, i.e. $s(\a) = \nangle{\a, \w}$, where $\w \in \R^d$ is an unknown direction in $\R^d$. We will also denote by $\a^* := \arg\max_{\a \in \cD}s(\a)$ the action with the highest score. 
The feedback model outputs a $1$-bit binary feedback $o \sim \text{Ber}(P(\a \succ \b))$ upon receiving any pair of actions $(\a,\b) \in \cD\times \cD$ -- we will use the abbreviation \lindb\ for this feedback model.

\paragraph{Linear Bandits (\linb) Feedback Model.} 
Similar to the \lindb\ setting above, we again assume a decision space $\cD\subset \R^d$ with each arm $\a \in \cD$ being associated to an underlying linear score/reward function $s(\a) = \nangle{\a, \w}$, $\w \in \R^d$ is unknown. As oppose to the preference feedback observed in the \lindb\ case, in this feedback model one gets to observe a noisy feedback of the true reward of any queried arm $\a \in \cD$ given by $r(\a) = s(\a) + \eta$, where $\eta$ is usually a $0$ mean noise. For example, $\eta \sim \text{Gaussian}(0,1)$ etc. We will use the abbreviation \linb\ for this feedback model.

\subsection{Simulating \lindb\ feedback with \linb\ Feedback.}
\label{sec:red}
Based on the reduction idea of \cite[Lemma 8]{saha2021optimal}, we know that one can always simulate $1$ unit of \lindb\ pairwise preference feedback from $2$ units of \linb\ reward/score feedback. Formally the claim is given by:

\begin{lemma}
\label{lem:red}
Given any two actions $\a$ and $\b \in \cD$ if $r(\a) = s(\a) + \eta_1$ and $r(\b) = s(\b) + \eta_2$ are two (noisy) \linb\ feedback, s.t. $\eta_1, \eta_2 \overset{\text{iid}}{\sim} \text{Gumbel(0,1)}$ are iid draws from Gumbel$(0,1)$ distribution, then the binary outcome $o = \mathbf 1(r(\a) > r(\b)) \sim P(\a \succ \b) $ follows \lindb \ feedback. 
\end{lemma}

Interested readers are encouraged to go over the proof of Lemma 9 of \cite{saha2021optimal} to see the proof of \cref{lem:red} above.
Given the above result, we are now ready to prove the lower bound results of \cref{thm:lbon} and \cref{thm:lboff} as discussed in the following two subsections. We would first define the online and offline versions of the best-arm identification problem (BAI) with \lindb\ feedback.

\begin{figure}[H]
	\begin{center}
\includegraphics[width=0.43\textwidth]{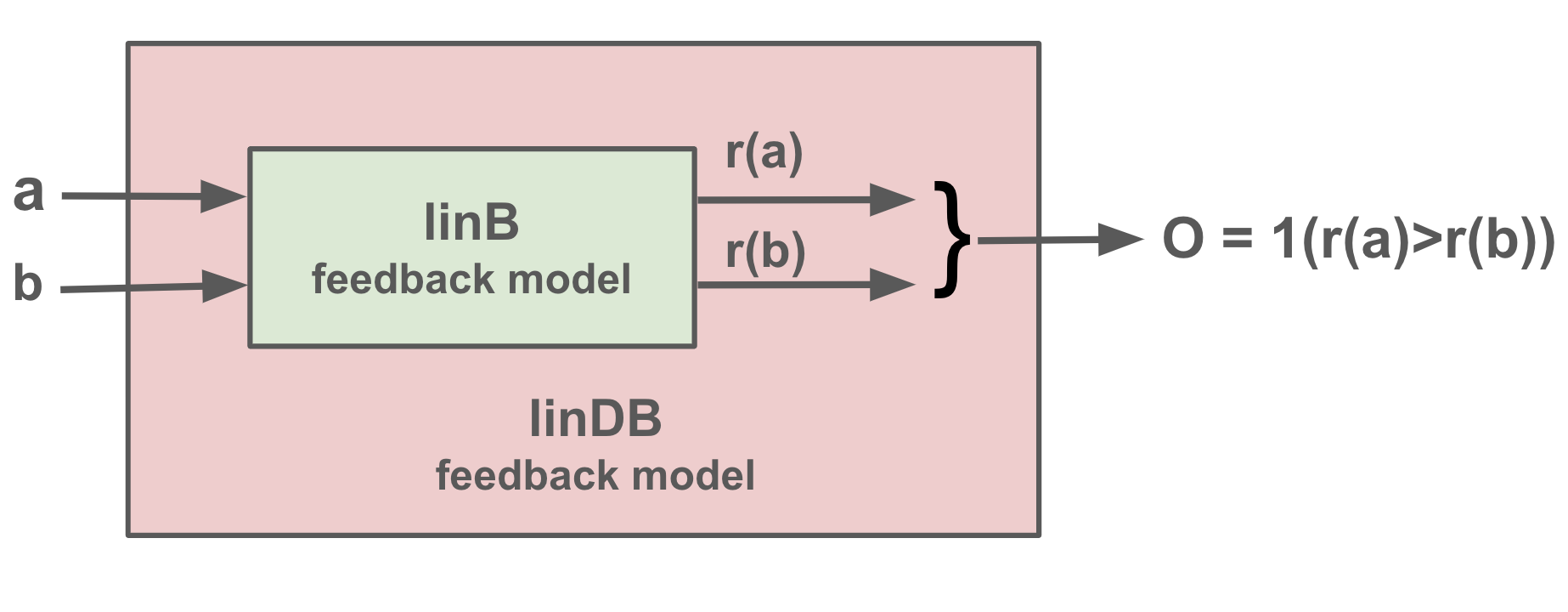} 
		\caption{Simulating \lindb\ feedback with \linb\ feedback}
		\label{fig:red}
	\end{center}
\end{figure}

\subsection{Proof of \cref{thm:lbon}}
\label{app:lbon}

The \emph{online version of BAI problem with \lindb\ feedback} is an active sequential decision-making process where at each round $t$ the learner (i.e. the algorithm) is supposed to play a pair of actions $(\a_t,\b_t) \in \cD \times \cD$, upon which it gets to see a binary preference feedback $o_t \sim P(\a_t \succ \b_t)$  is the preference feedback of the pair $(\a_t,\b_t)$ generated according to the \lindb\ feedback model. Given a horizon of $T$ rounds, the objective of the learner is to output an arm $\a_T \in \cD$ such that $\E[r(\a_T)] = s(\a_t)$ is maximized or the difference $s(\a^*) - s(\a_T)$ is minimized. Let us denote this problem as \emph{Online-BAI-}\lindb\ problem.

Similarly, the \emph{online version of BAI problem with \linb\ feedback} is an active sequential decision-making process where at each round $t$ the learner (i.e. the algorithm) is supposed to play an action $\a_t \in \cD$, upon which it gets to see a real-valued reward/score feedback $r_t$ generated according to the \linb\ feedback model. Given a horizon of $T$ rounds, the objective of the learner is to output an arm $\a_T \in \cD$ such that $\E[r(\a_T)] = s(\a_t)$ is maximized or the difference $s(\a^*) - s(\a_T)$ is minimized. Let us denote this problem as \emph{Online-BAI-}\linb\ problem.

\textbf{Key Idea: Reducing \emph{Online-BAI-}\linb\ to \emph{Online-BAI-}\lindb.} Owing to \cref{lem:red}, it is easy to see that one could reduce the \emph{Online-BAI-}\linb\ problem to \emph{Online-BAI-}\lindb\ problem, i.e. given any algorithm $\cA^{\lindb}$ for the latter problem, one can use it to solve the former. We described the idea in \cref{alg:red} below to construct $\cA^{\linb}$---an \emph{Online-BAI-}\linb\ algorithm using $\cA^{\lindb}$. 

\vspace{-5pt}
 \begin{center}
\begin{algorithm}[H]
   \caption{Simulating $\cA^{\linb}$ using $\cA^{\lindb}$ } 
   \label{alg:red}
\begin{algorithmic}[1]
\STATE \textbf{Input:} Time horizon $T$.
	\FOR {$t = 1, 2, \ldots \lceil\frac{T}{2}\rceil$}	
	  \STATE Receive: $(\a_t,\b_t) \leftarrow$ duel played by $\cA^{\lindb}$ at time $t$.
	  \STATE Play $\a_t$ at round $(2t-1)$. Receive $r(\a_t)$.
	  \STATE Play $\b_t$ at round $2t$. Receive  $r(\b_t)$.
	  \STATE Feedback: $o_t = \mathbf 1(r(\a_t) > r(\b_t))$ to $\cA^{\lindb}$.
	\ENDFOR   
\end{algorithmic}
\end{algorithm}
\end{center}
\vspace{-5pt}

But since a single round of \emph{Online-BAI-}\lindb\ corresponds to two rounds of \emph{Online-BAI-}\linb\, a valid 
BAI lower bound for the latter would immediately imply a BAI lower bound for the former problem. The result of \cref{thm:lbon} now follows from the known existing lower bound for the \emph{Online-BAI-}\linb\ problem from \citep[Theorem 2]{wagenmaker2022reward}.

\subsection{Proof of \cref{thm:lboff}}
\label{app:lboff}

The \emph{offline version of BAI problem with \lindb\ feedback} is a batch decision-making process where the learner is provided with an offline dataset $\cD_{\text{off}} = \{(\a_i,\b_i,o_i)\}_{i = 1}^n$, where for each triplet $(\a_i,\b_i,o_i)$, $(\a_i,\b_i) \in \cD\times \cD$ denotes a pair of action and $o_i \sim P(\a_i \succ \b_i)$ is the preference feedback of the pair $(\a_i,\b_i)$ generated according to the \lindb\ feedback model. The objective of the learner is to output an arm $\a_n \in \cD$ using the $n$ samples of $\cD_{\text{off}}$ such that $\E[r(\a_n)] = s(\a_t)$ is maximized or the difference $s(\a^*) - s(\a_n)$ is minimized. Let us denote this problem as \emph{Offline-BAI-}\lindb\ problem.

In the same spirit, the \emph{offline version of BAI problem with \lindb\ feedback} is a batch decision-making process where the learner is provided with an offline dataset $\cD_{\text{off}} = \{(\a_i,r_i)\}_{i = 1}^n$, where $\a_i \in \cD$ is an arbitrary action in $\cD$ and $r_i$ is a noisy sample of the score/reward of arm $i$ drawn according to the \linb\ feedback model. The objective of the learner is to output an arm $\a_n \in \cD$ using the $n$ samples of $\cD_{\text{off}}$ such that $\E[r(\a_n)] = s(\a_t)$ is maximized or the difference $s(\a^*) - s(\a_n)$ is minimized. Let us denote this problem as \emph{Offline-BAI-}\linb\ problem.

\textbf{Reducing \emph{Offline-BAI-}\linb\ to \emph{Offline-BAI-}\lindb.} Owing to \cref{lem:red} and following a similar approach described in \cref{alg:red} above it is easy to see that one can obtain an algorithm for the \emph{Offline-BAI-}\linb\ problem using an algorithm for the \emph{Offline-BAI-}\lindb\ algorithm. 


Following the same argument used for proving \cref{thm:lbon} in \cref{app:lbon} above, the result of \cref{thm:lboff} now follows from the known existing lower bound for the \emph{Offline-BAI-}\linb\ problem from \citep[Theorem 2]{li2022pessimism}.
\end{document}